%% file: naaclhlt2018.tex
\documentclass[11pt,a4paper]{article}
\usepackage[hyperref]{naaclhlt2018}
\usepackage{times}
\usepackage{latexsym}

\usepackage{url}
\usepackage{graphicx}
\usepackage{amsmath}
\usepackage{amsfonts}
\usepackage{amssymb}
\usepackage{nicefrac}
\usepackage[linesnumbered,ruled,vlined]{algorithm2e}
\usepackage{multirow}

\usepackage{amsthm}

\newtheorem{theorem}{Theorem}[section]
\newtheorem{thm}{Theorem}
\newtheorem{prop}{Proposition}

\aclfinalcopy 

\newcommand{\rM}{\mathrm{M}}

\def \B {\mathcal{B}}
\def \x {\mathbf{x}}
\def \H {\mathcal{H}}
\def \R {\mathbb{R}}

\def \sgn {\mbox{sgn}}
\def \a {\mathbf{a}}
\def \u {\mathbf{u}}

\def \E {\textbf{E}}

\def \A {\textbf{A}}

\def \y {\mathbf{y}}

\def \S {\textbf{S}}
\def \X {\textbf{X}}
\def \Q {\textbf{Q}}
\def \Z {\textbf{Z}}

\def \Y {\textbf{Y}}
\def \P {\mathcal{P}}
\def \u {\mathbf{u}}
\def \v {\mathbf{v}}
\def \e {\mathbf{e}}
\def \U {\textbf{U}}
\def \V {\textbf{V}}
\def \S {\textbf{S}}
\def \X {\textbf{X}}
\def \Q {\textbf{Q}}
\def \Z {\textbf{Z}}
\def \U {\textbf{U}}
\def \N {\textbf{N}}
\def \F {\textbf{F}}
\def \V {\textbf{V}}
\def \P {\textbf{P}}
\def \Y {\textbf{Y}}
\def \B {\textbf{B}}
\def \H {\textbf{H}}
\def \E {\textbf{E}}
\def \A {\textbf{A}}

\def \v {\mathbf{v}}
\def \e {\mathbf{e}}

\def \rank {\mbox{rank}}

\def \robusttc {\textsc{RobustTC}}

\title{Diverse Few-Shot Text Classification with Multiple Metrics}

\author{
 Mo Yu\thanks{Equal contributions from the corresponding authors: \texttt{yum@us.ibm.com, xiaoxiao.guo@ibm.com, jinfengy@us.ibm.com}.} \quad \quad \   Xiaoxiao Guo$^*$\quad \quad Jinfeng Yi$^*$\quad \quad Shiyu Chang \\ 
 \textbf{Saloni Potdar}\quad \textbf{Yu Cheng} \quad \textbf{Gerald Tesauro}\quad \textbf{Haoyu Wang}\quad \textbf{Bowen Zhou}\\
 \vspace{-0.2cm}\\
AI Foundations -- Learning,  IBM Research\\
IBM T. J. Watson Research Center, Yorktown Heights, NY 10598\\
}

\date{}

\begin{document}
\maketitle
\begin{abstract}
We study few-shot learning in natural language domains. Compared to many existing works that apply either metric-based or optimization-based meta-learning to image domain with low inter-task variance, we consider a more realistic setting, where tasks are diverse.  However, it imposes tremendous difficulties to existing state-of-the-art metric-based algorithms since a single metric is insufficient to capture complex task variations in natural language domain.   To alleviate the problem, we propose an adaptive metric learning approach that automatically determines the best weighted combination from a set of metrics obtained from meta-training tasks for a newly seen few-shot task.   Extensive quantitative evaluations on real-world sentiment analysis and dialog intent classification datasets demonstrate that the proposed method performs favorably against state-of-the-art few shot learning algorithms in terms of predictive accuracy. 
We make our code and data available for further study.\footnote{\tiny{\url{https://github.com/Gorov/DiverseFewShot_Amazon}}}
\end{abstract}

\section{Introduction}

Few-shot learning (FSL)~\cite{miller2000learning,li2006one,lake2015human} aims to learn classifiers from few examples per class.  Recently, deep learning has been successfully exploited for FSL via learning meta-models from a large number of \textbf{meta-training tasks}.  These meta-models can be then used for rapid-adaptation for the \textbf{target/meta-testing tasks} that only have few training examples.  Examples of such meta-models include: (1) metric-/similarity-based models, which learn contextual, and task-specific similarity measures  \cite{koch2015siamese,vinyals2016matching,snell2017prototypical}; and (2) optimization-based models, which receive the input of gradients from a FSL task and predict either model parameters or parameter updates~\citep{ravi2017optimization,munkhdalai2017meta,finn2017model,NIPS2017_7278}. 

In the past, FSL has mainly considered image domains, where all tasks are often sampled from one huge collection of data, such as Omniglot~\citep{lake2011one} and ImageNet~\citep{vinyals2016matching}, making tasks come from a single domain thus related.  Due to such a simplified setting, almost all previous works employ a common meta-model (metric-/optimization-based) for all few-shot tasks. However, this setting is far from the realistic scenarios in many real-world applications of few-shot text classification.  For example, on an enterprise AI cloud service, many clients submit various tasks to train text classification models for business-specific purposes. The tasks could be classifying customers' comments or opinions on different products/services, monitoring public reactions to different policy changes, or determining users' intents in different types of personal assistant services. As most of the clients cannot collect enough data, their submitted tasks form a few-shot setting. Also, these tasks are significantly diverse, thus a common metric is insufficient to handle all these tasks.

We consider a more realistic FSL setting where tasks are diverse.  In such a scenario, the optimal meta-model may vary across tasks.  Our solution is based on the metric-learning approach \citep{snell2017prototypical} and the key idea is to maintain multiple metrics for FSL.  The meta-learner selects and combines multiple metrics for learning the target task using \textbf{task clustering} on the meta-training tasks. During the meta-training, we propose to first partition the meta-training tasks into clusters, making the tasks in each cluster likely to be related.  Then within each cluster, we train a deep embedding function as the metric.  This ensures the common metric is only shared across tasks within the same cluster.  Further, during meta-testing, each target FSL task is assigned to a task-specific metric, which is a linear combination of the metrics defined by different clusters.  In this way, the diverse few-shot tasks can derive different metrics from the previous learning experience.

The key of the proposed FSL framework is the task clustering algorithm. Previous works~\citep{kumar2012learning,kang2011learning,crammer2012learning,barzilai2015convex} mainly focused on convex objectives, and assumed the number of classes is the same across different tasks (\emph{e.g.} binary classification is often considered). To make task clustering (i) compatible with deep networks and (ii) able to handle tasks with a various number of labels, we propose a \textbf{matrix-completion based task clustering} algorithm.   The algorithm utilizes task similarity measured by cross-task transfer performance, denoted by matrix $\S$.  The $(i,j)$-entry of $\S$ is the estimated accuracy by adapting the learned representations on the $i$-th (source) task to the $j$-th (target) task. We rely on matrix completion to deal with missing and unreliable entries in $\S$ and finally apply spectral clustering to generate the task partitions.

To the best of our knowledge, our work is the first one addressing the diverse few-shot learning problem and reporting results on real-world few-shot text classification problems.  The experimental results show that the proposed algorithm provides significant gains on few-shot sentiment classification and dialog intent classification tasks.  It provides positive feedback on the idea of using multiple meta-models (metrics) to handle diverse FSL tasks, as well as the proposed task clustering algorithm on automatically detecting related tasks.

\begin{figure*}[ht]
\centering
\includegraphics[scale=0.25]{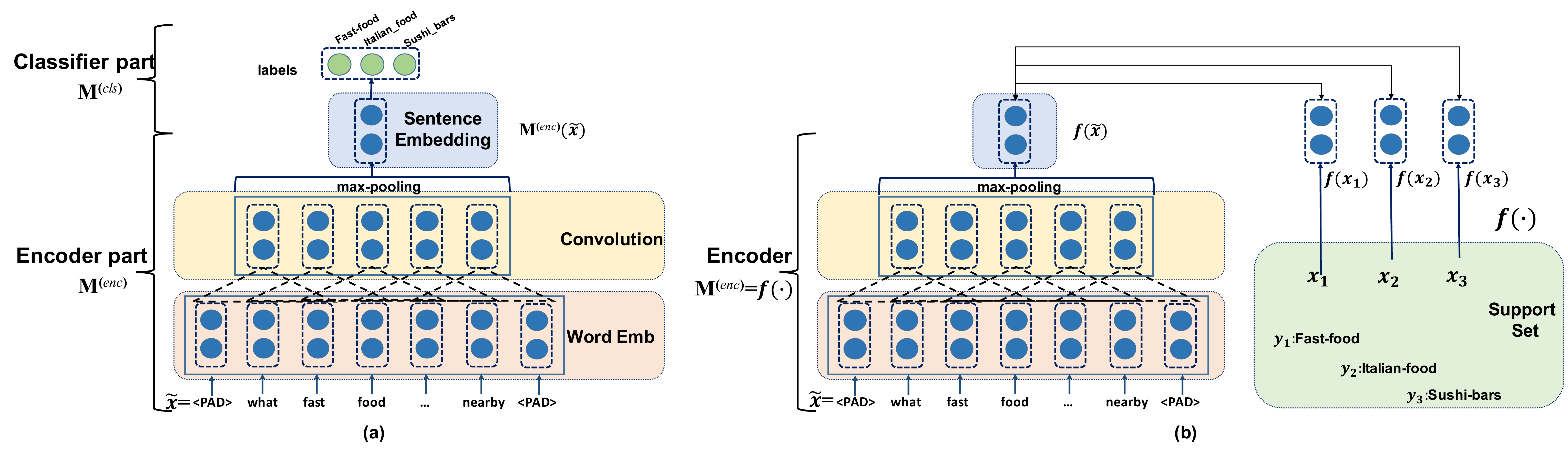}
\caption{{The Convolutional Neural Networks (CNN) used in this work: (a) A CNN classifier. The encoder component takes the sentence as input and outputs a fixed-length sentence embedding vector; the classifier component predicts class labels with the sentence embedding. (b) A Matching Network, which only contains an encoder like in (a), and makes prediction via a k-Nearest-Neighbor classifier with the similarity defined by the encoder.}
}
\label{fig:basic_models}
\end{figure*}

\section{Problem Definition}
\label{sec:notations}

\paragraph{Few-Shot Learning}
Since we focus on \textbf{diverse metric-based FSL}, the problem can be formulated in two stages:
(1) \textbf{meta-training}, where a set of metrics $\mathcal{M}=\left \{ \Lambda_1, \cdots, \Lambda_K \right \}$ is learned on the \textbf{meta-training tasks} $\mathcal{T}$.  Each $\Lambda_i$ maps two input $(x_1,x_2)$ to a scalar of similarity score. Here $\mathcal{T} = \left \{ \mathrm{T}_1, \mathrm{T}_2, \cdots, \mathrm{T}_N \right \}$ is a collection of $N$ tasks. 
Here $K$ is a pre-defined number (usually $K \ll N$).
Each task $\mathrm{T}_i$ consists of training, validation, and testing set denoted as $\left \{ D^{train}_{i}, D^{valid}_{i}, D^{test}_{i} \right\}$, respectively.  Note that the definition of $\mathcal{T}$ is a generalized version of $\mathcal{D}^{(meta-train)}$ in \citep{ravi2017optimization}, since each task $\mathrm{T}_i$ can be either few-shot (where $D^{valid}_{i}$ is empty) or regular\footnote{For example, the methods in \cite{triantafillou2017few} can be viewed as training meta-models from any sampled batches from one single meta-training dataset.}.
(2) \textbf{meta-testing}: the trained metrics in $\mathcal{M}$ is applied to \textbf{meta-testing tasks} denoted as $\mathcal{T'} = \left \{ \mathrm{T'}_1, \cdots, \mathrm{T'}_{N'} \right \}$, where each $\mathrm{T'}_i$ is a few-shot learning task consisting of both training and testing data as $\left \{ D'^{train}_{i}, D'^{test}_{i} \right\}$.  $D'^{train}_{i}$ is a small labeled set for generating the prediction model $\mathrm{M}'_i$ for each $\mathrm{T'_i}$. Specifically, $\mathrm{M}'_i$s are kNN-based predictors built upon the metrics in $\mathcal{M}$. We will detail the construction of $\mathrm{M}'_i$ in Section \ref{sec:methods}, Eq. (\ref{eqn:fsl}).  It is worth mentioning that the definition of $\mathcal{T}'$ is the same as $\mathcal{D}^{(meta-test)}$ in \citep{ravi2017optimization}.  The \textbf{performance of few-shot learning} is the macro-average of $\mathrm{M}'_i$'s accuracy on all the testing set $D'^{test}_i$s.

Our definitions can be easily generalized to other meta-learning approaches \cite{ravi2017optimization,finn2017model,mishra2017simple}.
The motivation of employing multiple metrics is that when the tasks are diverse, one metric model may not be sufficient.  Note that previous metric-based FSL methods can be viewed as a special case of our definition where $\mathcal{M}$ only contains a single $\Lambda$, as shown in the two base model examples below.

\paragraph{Base Model: Matching Networks}
In this paper we use the metric-based model Matching Network (MNet)~\cite{vinyals2016matching} as the base metric model.
The model (Figure \ref{fig:basic_models}b) consists of a neural network as the embedding function (\textbf{encoder}) and an augmented memory. The encoder, $f(\cdot)$, maps an input $x$ to a $d$-length vector. The learned metric $\Lambda$ is thus the similarity between the encoded vectors, $\Lambda(x_1,x_2)=f(x_1)^T f(x_2)$, i.e. the metric $\Lambda$ is modeled by the encoder $f$. The augmented memory stores a support set $S=\{(x_{i},y_{i})\}^{|S|}_{i=1}$, where $x_{i}$ is the supporting instance and $y_{i}$ is its corresponding label in a one-hot format. The MNet explicitly defines a classifier $\mathrm{M}$ conditioned on the supporting set $S$. For any new data $\hat{x}$, $\mathrm{M}$ predicts its label via a similarity function $\alpha(.,.)$ between the test instance $\hat{x}$ and the support set $S$:
\begin{align}
y = P(.|\hat{x}, S) = \sum_{i=1}^{|S|} \alpha(\hat{x}, x_{i};\theta) y_{i},
\label{eqn:base_mnet}
\end{align}
where we defined $\alpha(.,.)$ to be a softmax distribution given $\Lambda(\hat{x},x_i)$, where $x_{i}$ is a supporting instance, {\emph i.e.}, $\alpha(\hat{x}, x_{i};\theta) = \nicefrac{\exp(f(\hat{x})^T f(x_{i}))}{\sum_{j=1}^{|S|} \exp(f(\hat{x})^T f(x_{j}))}$, where $\theta$ are the parameters of the encoder $f$. Thus, $y$ is a valid distribution over the supporting set's labels $\{y_{i}\}_{i=1}^{|S|}$. To adapt the MNet to text classification, we choose encoder $f$ to be a convolutional neural network (CNN) following ~\cite{kim:2014:EMNLP2014,johnson2016supervised}.
Figure \ref{fig:basic_models} shows the MNet with the CNN architecture. 
Following \citep{collobert2011natural,kim:2014:EMNLP2014}, the model consists of a convolution layer and a max-pooling operation over the entire sentence. 

To train the MNets, we first sample the training dataset $D$ for task ${T}$ from all tasks $\mathcal{T}$, with notation simplified as $D\sim \mathcal{T}$. For each class in the sampled dataset $D$, we sample $k$ random instances in that class to construct a support set $S$, and sample a batch of training instances $B$ as training examples, i.e., $B,S\sim D$.  The training objective is to minimize the prediction error of the training samples given the supporting set (with regard to the encoder parameters $\theta$) as follows: 
\begin{equation}
\mathop{\mathbb{E}}_{D\sim \mathcal{T}} \Big[ \mathop{\mathbb{E}}_{B,S\sim D} \big[ \sum_{(x,y)\in B} \log(P(y|x,S;\theta))\big] \Big]. 
\label{eqn:mnet_obj}
\end{equation}

\paragraph{Base Model: Prototypical Networks}
Prototypical Network (ProtoNet)~\citep{snell2017prototypical} is a variation of Matching Network, which also depends on metric learning but builds the classifier $\mathrm{M}$ different from Eq. (\ref{eqn:base_mnet}):
\begin{equation}
y = P(.|\hat{x}, S) = \sum_{i=1}^{L} \alpha(\hat{x}, S_i;\theta) y_{i}.
\label{eqn:base_protonet}
\end{equation}
$L$ is the number of classes and $S_i\mathrm{=}\{x|(x,y) \in S\wedge y\mathrm{=}y_i\}$ is the support set of class $y_i$.
$\alpha(\hat{x}, S_{i};\theta) =\nicefrac{\exp\left( f(\hat{x})^T  \sum_{x\in S_i} f(x) \right)}{\sum_{j=1}^{L} \exp \left(f(\hat{x})^T \sum_{x' \in S_j} f(x') \right)}$.

\begin{figure*}[ht]
\centering
\includegraphics[scale=0.54]{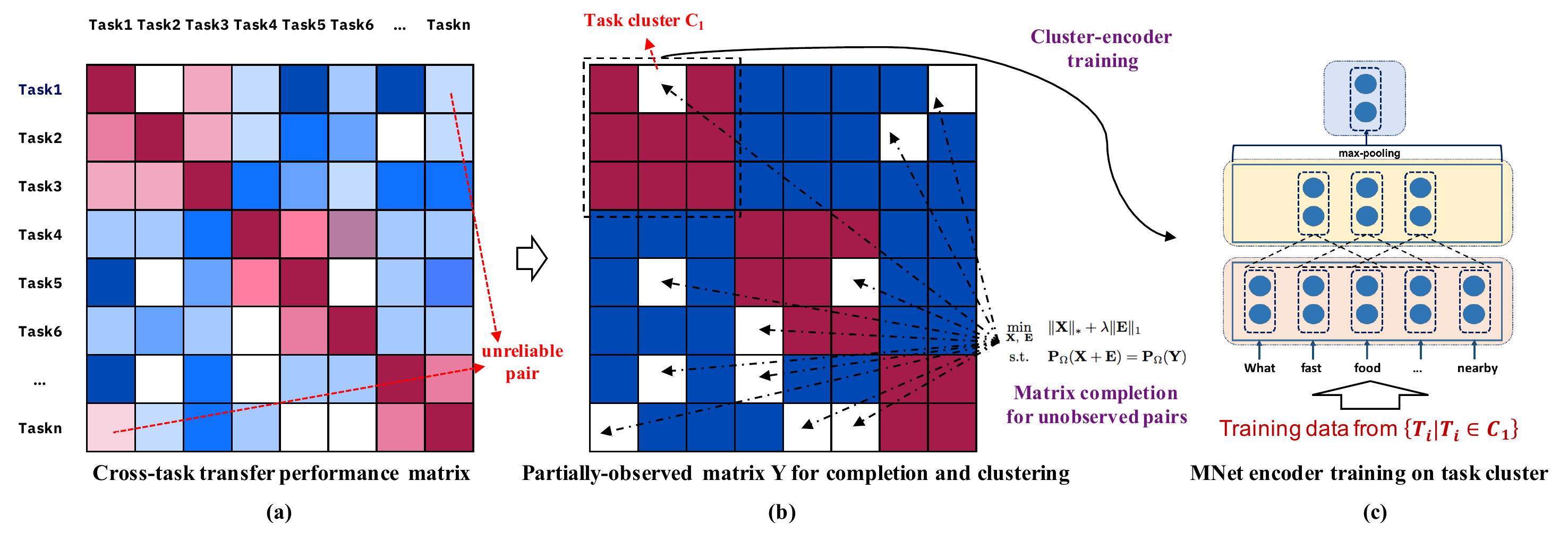}
\vspace{-0.1in}
\caption{{Overview of the idea of our multi-metric learning approach for few-shot learning. (a) an illustration of the sparse cross-tasks transfer-performance matrix with unobserved entries (white blocks) and unreliable values (top-right and bottom-left corners), where red colors indicate positive transfer and blue colors indicate negative transfer; 
(b) the constructed binary partially-observed matrix with low-rank constraint for matrix completion and clustering (see Section \ref{ssec:method_completion} for the details); (c) an encoder trained with the matching network objective Eq. (\ref{eqn:mnet_obj}) on a task cluster (tasks 1, 2 and 3 in the example).}
}
\label{fig:basic_idea}
\end{figure*}

\section{Methodology}
\label{sec:methods}

We propose a task-clustering framework to address the diverse few-shot learning problem stated in Section \ref{sec:notations}. We have the FSL algorithm summarized in Algorithm \ref{algo:taskcluster-fsl}. 
Figure \ref{fig:basic_idea} gives an overview of our idea.
The initial step of the algorithm is a novel task clustering algorithm based on matrix completion, which is described in Section \ref{ssec:method_completion}. The few-shot learning method based on task clustering is then introduced in Section \ref{ssec:method_fsl}.

\subsection{Robust Task Clustering by Matrix Completion}
\label{ssec:method_completion}

Our task clustering algorithm is shown in Algorithm \ref{algo:task-clustering}. The algorithm first evaluates the transfer performance by applying a single-task model $i$ to another task $j$ (Section \ref{sssec:encoder_transfer}), which will result in a (partially observed) cross-task transfer performance matrix $\textbf{S}$.
The matrix $\textbf{S}$ is then cleaned and completed, giving a symmetry task similarity matrix $\textbf{Y}$ for spectral clustering \cite{ng2002spectral}.

\subsubsection{Estimation of Cross-Task Transfer Performance}
\label{sssec:encoder_transfer}
Using single-task models, we can compute performance scores $s_{ij}$ by adapting each $\mathrm{M}_i$ to each task $T_j (j\neq i)$. This forms an 
$n \times n$ pair-wise classification performance matrix $\textbf{S}$,
called the \emph{transfer-performance matrix}.
Note that $\textbf{S}$ is asymmetric since usually $\S_{ij} \neq \S_{ji}$.

\begin{algorithm}[ht]
\small
{
\SetKwInOut{Input}{Input}
\SetKwInOut{Output}{Output}

 \Input{$N$ meta-training tasks $ \mathcal{T}$=$\left \{ \mathrm{T}_1, \mathrm{T}_2, \cdots, \mathrm{T}_n  \right \}$; number of clusters $K$; $N'$ target few-shot meta-testing tasks $\mathcal{T}'$}
 \Output{Meta-model $\mathcal{M} = \{ C_{1:K}\ (K\ \textrm{task clusters})$, $\mathcal{F} = \left \{ f_1,f_2, \cdots, f_K \right \} \ (K\ \textrm{task encoders})\}$ . One classifier $\mathrm{M'}_{i}$ for each target task $\mathrm{T}'$.}
 \DontPrintSemicolon
\BlankLine
 \textbf{Robust Task Clustering}: $C_{1:K}$ = \robusttc($\mathcal{T}$,$K$) (Algorithm \ref{algo:task-clustering}) \;
 \textbf{Cluster-Model Training}: Train one encoder (multi-task MNet) $f_i$ on each task cluster $C_i$ (Section \ref{sssec:method_meta_model})\;
 \textbf{Few-Shot Learning on Cluster-models}: Train a model $\mathrm{M}_{trg}$ on task $\mathrm{T}_{trg}$ with the method in Section \ref{sssec:method_fsl}.
 \caption{\label{algo:taskcluster-fsl}{\robusttc-FSL: Task Clustering for Few-Shot Learning}}}
\end{algorithm}

Ideally, the transfer performance could be estimated by training a MNet on task $i$ and directly evaluating it on task $j$. However, the limited training data usually lead to generally low transfer performance of single-task MNet.
As a result we adopt the following approach to estimate $\S$:

We train a CNN classifier (Figure \ref{fig:basic_models}(a)) on task $i$, then take only the encoder $\mathrm{M}^{enc}_i$ from $\rM_i$ and freeze it to train a classifier on task $j$.
This gives us a new task $j$ model, and we test this model on $D^{valid}_j$ to get the accuracy as the transfer-performance $\S_{ij}$.
The score shows how the representations learned on task $i$ can be adapted to task $j$, thus indicating the similarity between tasks.

\paragraph{Remark: Out-of-Vocabulary Problem}
In text classification tasks, transferring an encoder with fine-tuned word embeddings from one task to another is difficult as there can be a significant difference between the two vocabularies. Hence, while learning the single-task CNN classifiers, we always make the word embeddings fixed.

\setlength{\textfloatsep}{0pt}
\begin{algorithm}[t]
\small
{
\SetKwInOut{Input}{Input}
\SetKwInOut{Output}{Output}
 \Input{A set of $n$ tasks $ \mathcal{T} = \left \{ \mathrm{T}_1, \mathrm{T}_2, \cdots, \mathrm{T}_n  \right \}$, number of task clusters $K$}
 \Output{$K$ task clusters $C_{1:K}$}
 \DontPrintSemicolon
\BlankLine
 \textbf{Learning of Single-Task Models}: train single-task models $\mathrm{M}_i$ for each task $\mathrm{T}_i$\;
 \textbf{Evaluation of Transfer-Performance Matrix}: get performance matrix $\mathbf{\S}$ (Section \ref{sssec:encoder_transfer})\;
 \textbf{Score Filtering}: Filter the uncertain scores in $\S$ and construct the symmetric matrix $\Y$ using Eq. (\ref{eqn:A})\;
 \textbf{Matrix Completion}: Complete the similar matrix $\X$ from $\Y$ using Eq. (\ref{eqn:pro}) \;
 \textbf{Task Clustering}: $C_{1:K}$=SpectralClustering$(\X, K)$\;
 \caption{\label{algo:task-clustering}{\robusttc: Robust Task Clustering based on Matrix Completion}}}
\end{algorithm}

\subsubsection{Task Clustering Method}
\label{sssec:clustering_method}

Directly using the transfer performance for task clustering may suffer from both efficiency and accuracy issues. First, evaluation of all entries in the matrix $\S$ involves conducting the source-target transfer learning $O(n^2)$ times, where $n$ is the number of meta-training tasks. For a large number of diverse tasks where the $n$ can be larger than 1,000, evaluation of the full matrix is unacceptable (over 1M entries to evaluate). Second, the estimated cross-task performance (i.e. some $\S_{ij}$ or $\S_{ji}$ scores) is often unreliable due to small data size or label noise. When the number of the uncertain values is large, they can collectively mislead the clustering algorithm to output an incorrect task-partition. 
To address the aforementioned challenges, we propose a novel task clustering algorithm based on the theory of matrix completion~\citep{candes2010power}. 
Specifically, we deal with the huge number of entries by randomly sample task pairs to evaluate the $\S_{ij}$ and $\S_{ji}$ scores. Besides, we deal with the unreliable entries and asymmetry issue by keeping only task pairs $(i,j)$ with consistent $\S_{ij}$ and $\S_{ji}$ scores.
as will be introduced in Eq. (\ref{eqn:A}).
Below, we describe our method in detail. 

\paragraph{Score Filtering}
First, we use only reliable task pairs to generate a {\it partially-observed} similarity matrix $\Y$. Specifically, if $\S_{ij}$ and $\S_{ji}$ are high enough,  
then it is likely that tasks $\{i,j\}$ belong to a same cluster and share significant information. Conversely, if
$\S_{ij}$ and $\S_{ji}$ are low enough, 
then they tend to belong to different clusters. To this end, we need to design a mechanism to determine if a performance is high or low enough. Since different tasks may vary in difficulty, a fixed threshold is not suitable.
Hence, we define a dynamic threshold using the mean and standard deviation of the target task performance, i.e., $\mu_j = \text{mean}(\S_{:j})$ and $\sigma_j=\text{std}(\S_{:j})$, where $\S_{:j}$ is the $j$-th column of
$\S$. We then introduce two positive parameters $p_1$ and $p_2$, and define high and low performance as $\S_{ij}$ greater than $\mu_j + p_1 \sigma_j$ or lower than $\mu_j - p_2 \sigma_j$, respectively. When both $\S_{ij}$ and $\S_{ji}$ are high and low enough, we set their pairwise similarity as $1$ and $0$, respectively. Other task pairs are treated as uncertain task pairs and are marked as unobserved, and don't influence our clustering method.
This leads to a partially-observed symmetric matrix $\Y$, i.e., 
\begin{eqnarray}
\small
\Y_{ij}\mathrm{=}\Y_{ji}\mathrm{=}\left\{
\begin{array}{ll}
\multirow{2}{*}{1}     & \text{if}\ \ \S_{ij} > \mu_j + p_1 \sigma_j\ \ \\
&\text{and}\ \ \S_{ji} > \mu_i + p_1 \sigma_i\\
\multirow{2}{*}{0}     & \text{if}\ \ \S_{ij} < \mu_j - p_2 \sigma_j\ \ \\
&\text{and}\ \ \S_{ji} < \mu_i - p_2 \sigma_i\\
\mathrm{unobserved} & \mathrm{otherwise}
\end{array}
\right. \label{eqn:A}
\end{eqnarray}

\paragraph{Matrix Completion}
Given the partially observed matrix $\Y$, we then reconstruct the full similarity matrix $\X \in \mathbb{R}^{n\times n}$.
We first note that the similarity matrix $\X$ should be of low-rank (proof deferred to appendix). 
Additionally, since the observed entries of $\Y$ are generated based on high and low enough performance, it is safe to assume that most observed entries are correct and only a few may be incorrect. Therefore, we introduce a sparse matrix $\E$ to capture the observed incorrect entries in $\Y$. Combining the two observations, $\Y$ can be decomposed into the sum of two matrices $\X$ and $\E$, where $\X$ is a low rank matrix storing similarities between task pairs, and $\E$ is a sparse matrix that captures the errors in $\Y$.  The matrix completion problem can be cast as the following convex optimization problem:
\begin{eqnarray}\label{eqn:pro}
&\min\limits_{\X,\ \E} & \|\X\|_* + \lambda \|\E\|_1\\ \label{eqn:B}
& \mbox{s.t.}& \P_{\Omega}(\X+\E)  =  \P_{\Omega}(\Y), \nonumber
\end{eqnarray}
where $\|\circ\|_*$ denotes the matrix nuclear norm, the convex surrogate of rank function. $\Omega$ is the set of observed entries in $\Y$, and $\P_{\Omega}:\R^{n\times n} \mapsto \R^{n\times n}$ is a matrix projection operator defined as
\begin{eqnarray}
[\P_{\Omega}(\A)]_{ij} = \left\{
\begin{array}{ll}
\A_{ij} & \text{if}\ (i,j) \in \Omega \nonumber\\
0 & \mbox{otherwise}\nonumber
\end{array}
\right. \label{eqn:p}
\end{eqnarray}

Finally, we apply spectral clustering on the matrix $\X$ to get the task clusters.

\paragraph{Remark: Sample Efficiency}
In the Appendix A, we show a Theorem~\ref{thm:perfect-recovery} as well as its proof, implying that under mild conditions, the problem (\ref{eqn:pro}) can perfectly recover the underlying similarity matrix $\X^*$ if the number of observed correct entries is at least $O(n \log^2 n)$.
This theoretical guarantee implies that for a large number $n$ of training tasks, only a tiny fraction of all task pairs is needed to reliably infer similarities over all task pairs.

\subsection{Few-Shot Learning with Task Clusters}
\label{ssec:method_fsl}

\subsubsection{Training Cluster Encoders}
\label{sssec:method_meta_model}

For each cluster $C_k$, we train a multi-task MNet model (Figure~\ref{fig:basic_models}(b)) with all tasks in that cluster to encourage parameter sharing. 
The result, denoted as $f_k$ is called the \textbf{cluster-encoder} of cluster $C_k$. The $k$-th metric of the cluster is thus $\Lambda(x_1,x_2)=f_k(x_1)^{\intercal}f_k(x_2)$.

\subsubsection{Adapting Multiple Metrics for Few-Shot Learning}
\label{sssec:method_fsl}

To build a predictor $\mathrm{M}$ with access to only a limited number of training samples, we make the prediction probability by linearly combining prediction from learned cluster-encoders:
\begin{align}
p(y|x) = \sum_k \alpha_k P(y|x; f_k).
\label{eqn:fsl}
\end{align}
where $f_k$ is the learned (and frozen) encoder of the $k$-th cluster, $\{\alpha_{k}\}_{k=1}^{K}$ are adaptable parameters trained with few-shot training examples. 
And the predictor $P(y|x; f_k)$ from each cluster is 
\begin{eqnarray}
\small
P(y=y_l|x;f_k) = \frac{\exp\left \{ f_k (x_l)^{\intercal}f_k (x) \right \} }{\sum_{i} \exp \left \{ f_k (x_{i})^{\intercal}f_k (x) \right \} }
\end{eqnarray}
$x_{l}$ is the corresponding training sample of label $y_{l}$. 

\paragraph{Remark: Joint Method versus Pipeline Method}
End-to-end joint optimization on training data becomes a popular methodology for deep learning systems, but it is not directly applicable to diverse FSL. One main reason is that deep networks could easily fit any task partitions if we optimize on training loss only, making the learned metrics not generalize, as discussed in Section \ref{sec:related}.
As a result, this work adopts a pipeline training approach and employing validation sets for task clustering. Combining reinforcement learning with meta-learning could be a potential solution to enable an end-to-end training for future work.

\section{Tasks and Data Sets}

We test our methods by conducting experiments on two text classification data sets. We used NLTK toolkit\footnote{\url{http://www.nltk.org/}} for tokenization. The task are divided into meta-training tasks and meta-testing tasks (target tasks), where the meta-training tasks are used for clustering and cluster-encoder training. The meta-testing tasks are few-shot tasks, which are used for evaluating the method in Eq. (\ref{eqn:fsl}).

\subsection{Amazon Review Sentiment Classification}
First, following \citet{barzilai2015convex}, we construct multiple tasks with the multi-domain sentiment classification~\citep{blitzer2007biographies} data set. 
The dataset consists of Amazon product reviews for 23 types of products (see Appendix D for the details).
For each product domain, we construct three binary classification tasks with different thresholds on the ratings: the tasks consider a review as positive if it belongs to one of the following buckets $=5$ stars, $>=4$ stars or $>=2$ stars.\footnote{Data downloaded from \url{http://www.cs.jhu.edu/~mdredze/datasets/sentiment/}, in which the 3-star samples were unavailable due to their ambiguous nature \citep{blitzer2007biographies}.}
These buckets then form the basis of the task-setup, giving us 23 $\times$ 3$=$69 tasks in total. For each domain we distribute the reviews uniformly to the 3 tasks.
For evaluation, we select 12 (4$\times$3) tasks from 4 domains ({\it Books, DVD, Electronics, Kitchen}) as the meta-testing (target) tasks out of all 23 domains. For the target tasks, we create 5-shot learning problems. 

\subsection{Real-World Tasks: User Intent Classification for Dialog System}
The second dataset is from an online service which trains and serves intent classification models to various clients. The dataset comprises recorded conversations between human users and dialog systems in various domains, ranging from personal assistant to complex service-ordering or customer-service request scenarios. During classification, intent-labels\footnote{In conversational dialog systems, intent-labels are used to guide the dialog-flow.} are assigned to user utterances (sentences). We use a total of 175 tasks from different clients, and randomly sample 10 tasks from them as our target tasks. For each meta-training task, we randomly sample 64\% data into a training set, 16\% into a validation set, and use the rest as the test set.
The number of labels for these tasks varies a lot (from 2 to 100, see Appendix D for details), making regular $k$-shot settings not essentially limited-resource problems (e.g., 5-shot on 100 classes will give a good amount of 500 training instances). 
Hence, to adapt this to a FSL scenario, for target tasks we keep one example for each label (one-shot), plus 20 randomly picked labeled examples to create the training data. We believe this is a fairly realistic estimate of labeled examples one client could provide easily. 

\paragraph{Remark: Evaluation of the Robustness of Algorithm \ref{algo:task-clustering}} Our matrix-completion method could handle a large number of tasks via task-pair sampling. However, the sizes of tasks in the above two few-shot learning datasets are not too huge, so evaluation of the whole task-similarity matrix is still tractable. In our experiments, the incomplete matrices mainly come from the score-filtering step (see Eq. \ref{eqn:A}). Thus there is limited randomness involved in the generation of task clusters.

To strengthen the conclusion, we evaluate our algorithm on an additional dataset with a much larger number of tasks. The results are reported in the multi-task learning setting instead of the few-shot learning setting focused in this paper. Therefore we put the results to a non-archive version of this paper\footnote{\url{https://arxiv.org/pdf/1708.07918.pdf}}  for further reference.

\section{Experiments}

\subsection{Experiment Setup}
\label{ssec:exp_setup}

\paragraph{Baselines} 
We compare our method to the following baselines:
(1) \textbf{Single-task CNN}: training a CNN model for each task individually; (2) \textbf{Single-task FastText}: training one FastText model~\citep{joulin2016bag} with fixed embeddings for each individual task; 
(3) \textbf{Fine-tuned the holistic MTL-CNN}: a standard transfer-learning approach, which trains one MTL-CNN model on all the training tasks offline, then fine-tunes the classifier layer (i.e. $\mathrm{M}^{(cls)}$ Figure \ref{fig:basic_models}(a)) on each target task; (4) \textbf{Matching Network}: a metric-learning based few-shot learning model trained on all training tasks;
(5) \textbf{Prototypical Network}: a variation of matching network with different prediction function as Eq. \ref{eqn:base_protonet};
(6) \textbf{Convex combining all single-task models}: 
training one CNN classifier on each meta-training task individually and taking the encoder,
then for each target task training a linear combination of all the above single-task encoders with Eq. (\ref{eqn:fsl}).
This baseline can be viewed as a variation of our method without task clustering.
We initialize all models with pre-trained 100-dim Glove embeddings (trained on 6B corpus)~\citep{pennington2014glove}.

\begin{table*}[ht]
\centering
\begin{tabular}{l|cc}
\hline 
\multirow{2}{*}{\bf Model}  & \multicolumn{2}{c}{\bf Avg Acc} \\ 
& \bf Sentiment & \bf Intent \\
\hline
(1) Single-task CNN w/pre-trained emb   & 65.92 & 34.46 \\ 
(2) Single-task FastText w/pre-trained emb  & 63.05 & 23.87 \\ 
(3) Fine-tuned holistic MTL-CNN  & 76.56 & 30.36 \\
(4) Matching Network~\citep{vinyals2016matching} & 65.73 & 30.42  \\
(5) Prototypical Network~\citep{snell2017prototypical} & 68.15 & 31.51 \\
\hline
(6) Convex combination of all single-task models & 78.85 & 34.43 \\
\hline
\hline
{\bf \robusttc-FSL} & \bf 83.12 & \bf 37.59\\
\hline
{\bf Adaptive \robusttc-FSL} & - & {\bf 42.97}\\
\hline
\end{tabular}
\caption{Accuracy of FSL on sentiment classification (Sentiment) and dialog intent classification (Intent) tasks. The target tasks of sentiment classification are 5-shot ones; and each intent target task contains one training example per class and 20 random labeled examples.}
\label{tab:main_exp}
\vspace{-0.1 in}
\end{table*}

\paragraph{Hyper-Parameter Tuning}
In all experiments, we set both $p_1$ and $p_2$ parameters in (\ref{eqn:A}) to $0.5$.
This strikes a balance between obtaining enough observed entries in $\Y$, and ensuring that most of the retained similarities are consistent with the cluster membership.
The window/hidden-layer sizes of CNN and the initialization of embeddings (random or pre-trained) are tuned during the cluster-encoder training phase, with the validation sets of meta-training tasks. 
We have the CNN with window size of 5 and 200 hidden units. The single-metric FSL baselines have 400 hidden units in the CNN encoders. On sentiment classification, all cluster-encoders use random initialized word embeddings for sentiment classification, and use Glove embeddings as initialization for intent classification, which is likely because the training sets of the intent tasks are usually small. 

Since all the sentiment classification tasks are binary classification based on our dataset construction. A CNN classifier with binary output layer can be also trained as the cluster-encoder for each task cluster. Therefore we compared CNN classifier, matching network, and prototypical network on Amazon review, and found that CNN classifier performs similarly well as prototypical network. Since some of the Amazon review data is quite large which involves further difficulty on the computation of supporting sets, we finally use binary CNN classifiers as cluster-encoders in all the sentiment classification experiments.

Selection of the learning rate and number of training epochs for FSL settings, i.e., fitting $\alpha$s in Eq. (\ref{eqn:fsl}), is more difficult since there is no validation data in few-shot problems.
Thus we pre-select a subset of meta-training tasks as meta-validation tasks and tune the two hyper-parameters on the meta-validation tasks.

\subsection{Experimental Results}
\label{ssec:exp_main}
Table \ref{tab:main_exp} shows the main results on (i) the 12 few-shot product sentiment classification tasks by leveraging the learned knowledge from the 57 previously observed tasks from other product domains; and (ii) the 10 few-shot dialog intent classification tasks by leveraging the 165 previously observed tasks from other clients' data.

Due to the limited training resources, all the supervised-learning baselines perform poorly. 
The two state-of-the-art metric-based FSL approaches, matching network (4) and prototypical network (5), do not perform better compared to the other baselines, since the single metric is not sufficient for all the diverse tasks.
On intent classification where tasks are further diverse, all the single-metric or single-model methods (3-5) perform worse compared to the single-task CNN baseline (1).
The convex combination of all the single training task models is the best performing baseline overall. 
However, on intent classification it only performs on par with the single-task CNN (1), which does not use any meta-learning or transfer learning techniques, mainly for two reasons: (i) with the growth of the number of meta-training tasks, the model parameters grow linearly, making the number of parameters (165 in this case) in Eq.(\ref{eqn:fsl}) too large for the few-shot tasks to fit; (ii) the meta-training tasks in intent classification usually contain less training data, making the single-task encoders not generalize well.

In contrast, our \robusttc-FSL gives consistently better results compared to all the baselines.
It outperforms the baselines in previous work (1-5) by a large margin of more than 6\% on the sentiment classification tasks, and more than 3\% on the intent classification tasks.
It is also significantly better than our proposed baseline (6), showing the advantages of the usage of task clustering.

\paragraph{Adaptive \robusttc-FSL}
Although the \robusttc-FSL improves over baselines on intent classification, the margin is smaller compared to that on sentiment classification, because the intent classification tasks are more diverse in nature. This is also demonstrated by the training accuracy on the target tasks, where several tasks fail to find any cluster that could provide a metric that suits their training examples. 
To deal with this problem, we propose an improved algorithm to automatically discover whether a target task belongs to none of the task-clusters. If the task doesn't belong to any of the clusters, it cannot benefit from any previous knowledge thus falls back to single-task CNN. The target task is treated as ``out-of-clusters'' when none of the clusters could achieve higher than 20\% accuracy (selected on meta-validation tasks) on its training data. We call this method \textbf{Adaptive \robusttc-FSL}, which gives more than 5\% performance boost over the best \robusttc-FSL result on intent classification. Note that the adaptive approach makes no difference on the sentiment tasks, because they are more closely related so re-using cluster-encoders always achieves better results compared to single-task CNNs.

\subsection{Analysis}

\paragraph{Effect of the number of clusters}
Figure \ref{fig:exp_cluster} shows the effect of cluster numbers on the two tasks. \robusttc\ achieves best performance with 5 clusters on sentiment analysis (SA) and 20 clusters on intent classification (Intent). All clustering results significantly outperform the single-metric baselines (\#cluster=1 in the figure).

\begin{figure}[t]
\centering
\includegraphics[scale=0.36]{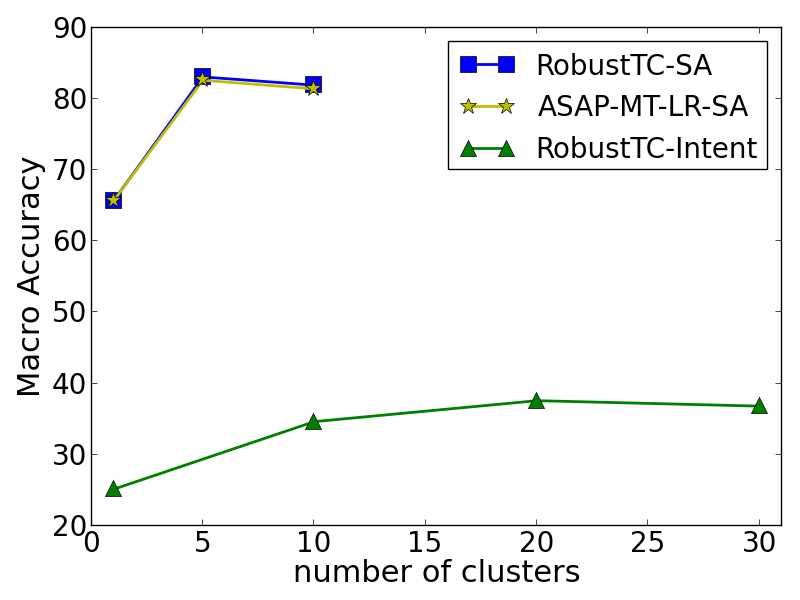}
\caption{Effect of clusters. \robusttc-SA and \robusttc-Intent: the performance of our \robusttc\ clusters on the sentiment and intent classification tasks. ASAP-MT-LR-SA: the state-of-the-art ASAP-MT-LR clusters on the sentiment-analysis tasks (the method is not applicable to the intent-classification tasks).
}
\vspace{1.5 em}
\label{fig:exp_cluster}
\end{figure}

\begin{table*}[ht]
\centering
\tiny
\begin{tabular}{c|c|c|c|c|c|c|c|c|c|c}
\hline
&{\bf Clus0}& {\bf Clus1}& {\bf Clus2}& {\bf Clus3}& {\bf Clus4}& {\bf Clus5}& {\bf Clus6}& {\bf Clus7}& {\bf Clus8}& {\bf Clus9}\\ \hline
& automotive.t2& apparel.t2& baby.t5& automotive.t5& apparel.t5& beauty.t4& camera.t4& gourmet.t5& cell.t4& apparel.t4\\
& camera.t2& automotive.t4& magazines.t5& baby.t4& camera.t5& beauty.t5& software.t2& magazines.t4& software.t5& toys.t2\\
& health.t2& baby.t2& sports.t5& health.t4& grocery.t5& cell.t5& software.t4& music.t4& toys.t4& \\
& magazines.t2& cell.t2& toys.t5& health.t5& jewelry.t5& gourmet.t2& & music.t5& & \\
& office.t2& computer.t2& video.t5& & & gourmet.t4& & video.t4& & \\
& outdoor.t2& computer.t4& & & & grocery.t2& & & & \\
& sports.t2& computer.t5& & & & grocery.t4& & & & \\
& sports.t4& jewelry.t4& & & & office.t4& & & & \\
& & music.t2& & & & outdoor.t4& & & & \\
& & video.t2& & & & & & & & \\
\hline
\bf dvd-t4& 0.4844 & 0.4416 & 0.4625 & \textcolor{blue}{\bf 0.7843} & \textcolor{blue}{\bf 0.7970} & 0.7196 & \textcolor{blue}{\bf 0.8952} & 0.3763 & 0.7155 & 0.6315\\
\bf dvd-t5& 0.0411 & -0.2493 & \textcolor{blue}{\bf 0.5037} & \textcolor{blue}{\bf 0.3567} & 0.1686 & -0.0355 & \textcolor{blue}{\bf 0.4150} & -0.2603& -0.0867 & 0.0547\\
\bf kitchen-t4&  0.6823 & 0.7268 & 0.7929 & \textcolor{blue}{\bf 1.2660} & \textcolor{blue}{\bf 1.1119} & 0.7255 & \textcolor{blue}{\bf 1.2196} & 0.7065 & 0.6625 & 1.0945\\
\hline
\end{tabular}
\caption{Visualization of clusters on the Amazon review domain. The top shows the training tasks assigned to the 10 clusters. Here the number N$\in\{2,4,5\}$ refers to the threshold of stars for positive reviews. 
At the bottom we show three tasks with largest improvement from \robusttc-FSL. The top-3 most relevant task clusters (i.e. with highest weights $\alpha$s in Eq.\ref{eqn:fsl} ) are highlighted with \textcolor{blue}{\bf blue bold} font.}\label{tab:cluseter_visual}
\vspace{-0.1in}
\end{table*}

\paragraph{Effect of the clustering algorithms}
Compared to previous task clustering algorithms, our \robusttc\ is the only one that can cluster tasks with varying numbers of class labels (e.g. in intent classification tasks).
Moreover, we show that even in the setting of all binary classifications tasks (e.g. the sentiment-analysis tasks) that previous task clustering research work on, our \robusttc\ is still slightly better for the diverse FSL problems.
Figure \ref{fig:exp_cluster} compares with a state-of-the-art logistic regression based task clustering method (\textbf{ASAP-MT-LR})~\citep{barzilai2015convex}. 
Our \robusttc\ clusters give slightly better FSL performance (e.g. 83.12 vs. 82.65 when \#cluster=5).

\paragraph{Visualization of Task Clusters}
The top rows of Table \ref{tab:cluseter_visual} shows the ten clusters used to generate the sentiment classification results in Figure \ref{fig:exp_cluster}. From the results, we can see that tasks with same thresholds are usually grouped together; and tasks in similar domains also tend to appear in the same clusters, even the thresholds are slightly different (e.g. t2 vs t4 and t4 vs t5).

The bottom of the table shows the weights $\alpha$s in Eq.(\ref{eqn:fsl}) for the target tasks with the largest improvement. It confirms that our \robusttc-FSL algorithm accurately adapts multiple metrics for the target tasks.

\section{Related Work}
\label{sec:related}

{\noindent \bf Few Shot Learning} \quad
FSL~\citep{miller2000learning,li2006one,lake2015human} aims to learn classifiers for new classes with only a few training examples per class. 
Recent deep learning based FSL approaches mainly fall into two categories:
(1) \emph{metric-based approaches} \cite{koch2015siamese,vinyals2016matching,snell2017prototypical}, which aims to learn 
generalizable metrics and corresponding matching functions from multiple training tasks.
These approaches essentially learn one metric for all tasks, which is sub-optimal when the tasks are diverse.
(2) \emph{optimization-based approaches}~\cite{ravi2017optimization,munkhdalai2017meta,finn2017model}, which aims to learn to optimize model parameters (by either predicting the parameter updates or directly predicting the model parameters) given the gradients computed from few-shot examples.



Previous FSL research usually adopts the $k$-shot, $N$-way setting, where all the few-shot tasks have the same number of $N$ class labels, and each label has $k$ training instances.
Moreover, these few-shot tasks are usually constructed by sampling from one huge dataset, thus all the tasks are guaranteed to be related to each other.
However, in real-world applications, the few-shot learning tasks could be diverse: there are different tasks with varying number of class labels and they are not guaranteed to be related to each other.
As a result, a single meta-model or metric-model is usually not sufficient to handle all the few-shot tasks.

{\noindent \bf Task Clustering} \quad
Previous task clustering methods measure the
task relationships in terms of similarities among single-task model parameters~\citep{kumar2012learning,kang2011learning};
or jointly assign task clusters and train model parameters for each cluster to minimize the overall training loss \citep{crammer2012learning,barzilai2015convex,murugesan2017co}.
These methods usually work on convex models but do not fit the deep networks, mainly because of (i) the parameters of deep networks are very high-dimensional and their similarities are not necessarily related to the functional similarities; and (ii) deep networks have flexible representation power so they may overfit to arbitrary cluster assignment if we consider training loss alone.
Moreover, these methods require identical class label sets across different tasks, which does not hold in most of the realistic settings. 

\section{Conclusion}
We propose a few-shot learning approach for diverse tasks based on task clustering.
The proposed method can use multiple metrics, and performs significantly better compared to previous single-metric methods when the few-shot tasks come from diverse domains.
Future work includes applying the task-clustering idea to other FSL algorithms \cite{ravi2017optimization,finn2017model,cheng2017metametric}, and exploring more advanced composition methods of cluster-encoders beyond linear combination~\cite{chang2013multimedia,andreas2016neural}. 

\bibliography{naaclhlt2018}
\bibliographystyle{acl_natbib}
\clearpage

\input{app.tex}

\end{document}

%% file: app.tex
\onecolumn

\section*{Appendix A: Perfect Recovery Guarantee for the Problem (\ref{eqn:pro})}
The following theorem shows the perfect recovery guarantee for the problem (\ref{eqn:pro}). Appendix C provides the proof for completeness.

\begin{theorem}\label{thm:perfect-recovery}
Let $\X^* \in \R^{n\times n}$ be a rank $k$ matrix with a singular value decomposition  $\X^* = \U\Sigma \V^{\top}$, where $\U = (\u_1, \ldots, \u_k) \in \R^{n\times k}$ and $\V = (\v_1, \ldots, \v_k) \in \R^{n\times k}$ are the left and right singular vectors of $\X^*$, respectively. Similar to many related works of matrix completion, we assume that the following two assumptions are satisfied:
\begin{enumerate}
\item The row and column spaces of $\X$ have coherence bounded above by a positive number $\mu_0$.
\item Max absolute value in matrix\ $\U\V^{\top}$ is bounded above by $\mu_1\sqrt{r}/n$ for a positive number $\mu_1$.
\end{enumerate}
Suppose that $m_1$ entries of $\X^*$ are observed with their locations sampled uniformly at random, and among the $m_1$ observed entries, $m_2$ randomly sampled entries are corrupted. Using the resulting partially observed matrix as the input to the problem (\ref{eqn:pro}), then with a probability at least $1 - n^{-3}$, the underlying matrix $\X^*$ can be perfectly recovered, given 
\begin{enumerate}
\item $\mu(\E)\xi(\X) \leq \frac{1}{4k + 5}$,
\item $\frac{\xi(\X) - (2k -1)\mu(\E)\xi^2(\X)}{1 - 2(k+1)\mu(\E)\xi(\X)} < \lambda < \frac{1 - (4k+5)\mu(\E)\xi(\X)}{(k+2)\mu(\E)}$,
\item $\ m_1 - m_2 \geq C[\max(\mu_0, \mu_1)]^4n\log^2 n$,
\end{enumerate}
where $C$ is a positive constant; $\xi(\circ)$ and $\mu(\circ)$ denotes the low-rank and sparsity incoherence~\citep{chandrasekaran2011rank}.
\end{theorem}
Theorem~\ref{thm:perfect-recovery} implies that even if some of the observed entries computed by (\ref{eqn:A}) are incorrect, problem (\ref{eqn:pro}) can still perfectly recover the underlying similarity matrix $\X^*$ if the number of observed correct entries is at least $O(n \log^2 n)$. 
For MATL with large $n$, this implies that only a tiny fraction of all task pairs is needed to reliably infer similarities over all task pairs.
Moreover, the completed similarity matrix $\X$ is symmetric, due to symmetry of the  input matrix $\Y$.
This enables analysis by similarity-based clustering algorithms, such as spectral clustering.

\section*{Appendix B: Proof of Low-rankness of Matrix $\X$}
We first prove that the full similarity matrix $\X \in \mathbb{R}^{n\times n}$ is of low-rank. To see this, let  $\A = (\a_1, \ldots, \a_k)$ be the underlying perfect clustering result, where $k$ is the number of clusters and $\a_i \in \{0, 1\}^n$ is the membership vector for the $i$-th cluster. Given $\A$, the similarity matrix $\X$ is computed as
\[
    \X = \sum_{i=1}^k \a_i \a_i^{\top} = \sum_{i=1}^k \B_i
\]
where $\B_i = \a_i \a_i^{\top}$ is a rank one matrix. Using the fact that $\mbox{rank}(\X) \leq \sum_{i=1}^k \mbox{rank}(\B_i)$ and $\mbox{rank}(\B_i) =1$, we have $\mbox{rank}(\X) \leq k$, i.e., the rank of the similarity matrix $\X$ is upper bounded by the number of clusters. Since the number of clusters is usually small,
the similarity matrix $\X$ should be of low rank. 

\section*{Appendix C: Proof of Theorem \ref{thm:perfect-recovery}}
We then prove our main theorem. First, we define several notations that are used throughout the proof. Let $\X = \U\Sigma \V^{\top}$ be the singular value decomposition of matrix $\X$, where $\U = (\u_1, \ldots, \u_k) \in \R^{n\times k}$ and $\V = (\v_1, \ldots, \v_k) \in \R^{n\times k}$ are the left and right singular vectors of matrix $\X$, respectively. Similar to many related works of matrix completion, we assume that the following two assumptions are satisfied:
\begin{enumerate}
\item {\bf A1}: the row and column spaces of $\X$ have coherence bounded above by a positive number $\mu_0$, i.e., $\sqrt{n/r} \max_{i}\|\P_{\U}(\e_i)\| \leq \mu_0$ and $\sqrt{n/r} \max_{i}\|\P_{\V}(\e_i)\| \leq \mu_0$, where $\P_{\U} = \U\U^{\top}$, $\P_{\V} = \V\V^{\top}$, and $\e_i$ is the standard basis vector, and
\item {\bf A2}: the matrix $\U\V^{\top}$ has a maximum entry bounded by $\mu_1\sqrt{r}/n$ in absolute value for a positive number $\mu_1$.
\end{enumerate}
Let $T$ be the space spanned by the elements of the form $\u_i\y^{\top}$ and $\x\v^{\top}_i$, for $1 \leq i \leq k$, where $\x$ and $\y$ are arbitrary $n$-dimensional vectors. Let $T^{\perp}$ be the orthogonal complement to the space $T$, and let $\P_T$ be the orthogonal projection onto the subspace $T$ given by
\[
    \P_T(\Z) = \P_{\U}\Z + \Z\P_{\V} - \P_{\U}\Z\P_{\V}.
\]

The following proposition shows that for any matrix $\Z \in T$, it is a zero matrix if enough amount of its entries are zero.
\begin{prop}
Let $\Omega$ be a set of $m$ entries sampled uniformly at random from $[1,\ldots, n]\times[1,\ldots, n]$, and $\P_{\Omega}(\Z)$ projects matrix $\Z$ onto the subset $\Omega$. If $m > m_0$, where $m_0 = C_R^2\mu_0rn\beta\log n$ with $\beta > 1$ and $C_R$ being a positive constant, then for any $\Z \in T$ with $\P_{\Omega}(\Z) = 0$, we have $\Z = 0$ with probability $1 - 3n^{-\beta}$.
\end{prop}
\begin{proof}
According to the Theorem 3.2 in \cite{candes2010power}, for any $\Z \in T$, with a probability at least $1 - 2n^{2 - 2\beta}$, we have
\begin{eqnarray}\label{eqn:4}
 \|\P_T(\Z)\|_F - \delta \|\Z\|_F \leq \frac{n^2}{m}\|\P_T\P_{\Omega}\P_T(\Z)\|_F^2 = 0
\end{eqnarray}
where $\delta = m_0/m < 1$. Since $\Z \in T$, we have $P_T(\Z) = \Z$. Then from (\ref{eqn:4}), we have $\|\Z\|_F \leq 0 $ and thus $\Z = 0$.
\end{proof}

In the following, we will develop a theorem for the dual certificate that guarantees the unique optimal solution to the following optimization problem
\begin{eqnarray}\label{eqn:pro_appendix}
&\min\limits_{\X,\ \E} & \|\X\|_* + \lambda \|\E\|_1\\ \label{eqn:B_app}
& \mbox{s.t.}& \P_{\Omega}(\X+\E)  =  \P_{\Omega}(\Y). \nonumber
\end{eqnarray}
\begin{thm}\label{thm:1}
Suppose we observe $m_1$ entries of $\X$ with locations sampled uniformly at random, denoted by $\Omega$. We further assume that $m_2$ entries randomly sampled from $m_1$ observed entries are corrupted, denoted by $\Delta$. Suppose that $\P_{\Omega}(\Y) = \P_{\Omega}(\X + \E)$ and the number of observed correct entries $m_1 - m_2 > m_0=C_R^2\mu_0rn\beta\log n$. Then, for any $\beta > 1$, with a probability at least $1 - 3n^{-\beta}$, the underlying true matrices $(\X, \E)$ is the unique optimizer of (\ref{eqn:pro_appendix}) if both assumptions {\bf A1} and {\bf A2} are satisfied and there exists a dual $\Q \in \R^{n\times n}$ such that (a) $\Q = \P_{\Omega}(\Q)$, (b) $\P_T(\Q) = \U\V^{\top}$, (c) $\|\P_{T^{\top}}(\Q)\| < 1$, (d) $\P_{\Delta}(\Q) = \lambda\ \sgn(\E)$, and (e) $\|\P_{\Delta^c}(\Q)\|_{\infty} < \lambda$.
\end{thm}
\begin{proof}
First, the existence of $\Q$ satisfying the conditions (a) to (e) ensures that $(\X, \E)$ is an optimal solution. We only need to show its uniqueness and we prove it by contradiction. Assume there exists another optimal solution $(\X+\N_\X, \E+\N_\E)$, where $\P_{\Omega}(\N_\X + \N_\E) = 0$. Then we have
\begin{eqnarray*}
\|\X+\N_\X\|_* + \lambda \|\E+\N_\E\|_1 & \geq & \|\X\|_* +\lambda \|\E\|_1 +  \langle \Q_\E, \N_\E \rangle + \langle \Q_\X, \N_\X \rangle
\end{eqnarray*}
where $\Q_\E$ and $\Q_\X$ satisfying $\P_{\Delta}(\Q_\E) = \lambda\ \sgn(\E)$, $\|\P_{\Delta^c}(\Q_\E)\|_{\infty} \leq \lambda$, $\P_T(\Q_\X) = \U\V^{\top}$ and $\|\P_{T^{\perp}}(\Q_\X)\| \leq 1$. As a result, we have
\begin{eqnarray*}
& & \lambda \|\E+\N_\E\|_1 + \|\X+\N_\X\|_* \\
& \geq & \lambda \|\E\|_1 + \|\X\|_* + \langle \Q + \P_{\Delta^c}(\Q_\E) - \P_{\Delta^c}(\Q), \N_\E \rangle + \langle \Q + \P_{T^{\perp}}(\Q_\X) - \P_{T^{\perp}}(\Q), \N_\X \rangle \\
& = & \lambda \|\E\|_1 + \|\X\|_*  + \langle \Q, \N_\E + \N_\X \rangle + \langle \P_{\Delta^c}(\Q_\E) - \P_{\Delta^c}(\Q), \N_\E \rangle + \langle \P_{T^{\perp}}(\Q_\X) - \P_{T^{\perp}}(\Q), \N_\X \rangle \\
& = & \lambda \|\E\|_1 + \|\X\|_*  + \langle \P_{\Delta^c}(\Q_\E) - \P_{\Delta^c}(\Q), \P_{\Delta^c}(\N_\E) \rangle + \langle \P_{T^{\perp}}(\Q_\X) - \P_{T^{\perp}}(\Q), \P_{T^{\perp}}(\N_\X) \rangle
\end{eqnarray*}

We then choose $\P_{\Delta^c}(\Q_\E)$ and $\P_{T^{\perp}}(\Q_\X)$ to be such that $\langle \P_{\Delta^c}(\Q_\E), \P_{\Delta^c}(\N_\E) \rangle = \lambda \|\P_{\Delta^c}(\N_\E)\|_1$ and $\langle \P_{T^{\perp}}(\Q_\X), \P_{T^{\perp}}(\N_\X) \rangle = \|\P_{T^{\perp}}(\N_\X)\|_{*}$. We thus have
\begin{eqnarray*}
& & \lambda \|\E+\N_\E\|_1 + \|\X+\N_\X\|_*\\
& \geq &\lambda \|\E\|_1 + \|\X\|_*  + (\lambda - \|\P_{\Delta^c}(\Q)\|_{\infty}) \|\P_{\Delta^c}(\N_\E)\|_1 + (1 - \|\P_{T^{\perp}}(\Q)\|)\|\P_{T^{\perp}}(\N_\X)\|_{*}
\end{eqnarray*}
Since $(\X+\N_\X, \E+\N_\E)$ is also an optimal solution, we have $\|\P_{\Omega^c}(\N_E)\|_1 = \|\P_{T^{\perp}}(\N_\X)\|_{*}$, leading to $\P_{\Omega^c}(\N_\E) = \P_{T^{\perp}}(\N_\X) = 0$, or $\N_\X \in T$. Since $\P_{\Omega}(\N_\X + \N_\E) = 0$, we have $\N_\X = \N_\E + \Z$, where $P_{\Omega}(\Z) = 0$ and $\P_{\Omega^c}(\N_\E) = 0$. Hence, $\P_{\Omega^c \cap \Omega}(\N_\X) = 0$, where $|\Omega^c \cap \Omega| = m_1 - m_2$. Since $m_1 - m_2 > m_0$, according to Proposition 1, we have, with a probability $1 - 3n^{-\beta}$, $\N_\X = 0$. Besides, since $\P_{\Omega}(\N_\X+\N_\E) = \P_{\Omega}(\N_\E) = 0$ and $\Delta \subset \Omega$, we have $\P_{\Delta}(\N_\E) = 0$. Since $\N_\E = \P_{\Delta}(\N_\E) + \P_{\Delta^c}(\N_\E)$, we have $\N_\E = 0$, which leads to the contradiction.
\end{proof}

Given Theorem~\ref{thm:1}, we are now ready to prove Theorem 3.1.
\begin{proof}
The key to the proof is to construct the matrix $\Q$ that satisfies the conditions (a)-(e) specified in Theorem~\ref{thm:1}. First, according to Theorem~\ref{thm:1}, when $m_1 - m_2 > m_0=C_R^2\mu_0rn\beta\log n$, with a probability at least $1 - 3n^{-\beta}$, mapping $\P_T\P_{\Omega}\P_T(\Z): T \mapsto T$ is an one to one mapping and therefore its inverse mapping, denoted by $(\P_T\P_{\Omega}\P_T)^{-1}$ is well defined. Similar to the proof of Theorem 2 in~\cite{chandrasekaran2011rank}, we construct the dual certificate $\Q$ as follows
\[
    \Q = \lambda\ \sgn(\E) + \epsilon_{\Delta} + \P_{\Delta}\P_T(\P_T\P_{\Omega}\P_T)^{-1}(\U\V^{\top} + \epsilon_T)
\]
where $\epsilon_T \in T$ and $\epsilon_{\Delta} = \P_{\Delta}(\epsilon_{\Delta})$. We further define
\begin{eqnarray*}
\H & = & \P_{\Omega}\P_T(\P_T\P_{\Omega}\P_T)^{-1}(\U\V^{\top}) \\
\F & = & \P_{\Omega}\P_T(\P_T\P_{\Omega}\P_T)^{-1}(\epsilon_{T})
\end{eqnarray*}
Evidently, we have $\P_{\Omega}(\Q) = \Q$ since $\Delta \subset \Omega$, and therefore the condition (a) is satisfied. To satisfy the conditions (b)-(e), we need
\begin{eqnarray}
\P_T(\Q) = \U\V^{\top} & \rightarrow & \epsilon_T = -\P_T(\lambda\ \sgn(\E) + \epsilon_{\Delta}) \label{eqn:c1}\\
\|\P_{T^{\perp}}(\Q)\| < 1 & \rightarrow & \mu(\E)\left(\lambda + \|\epsilon_{\Delta}\|_{\infty}\right) + \|\P_{T^{\perp}}(\H)\| + \|\P_{T^{\perp}}(\F)\| < 1\label{eqn:c2} \\
\P_{\Delta}(\Q) = \lambda\ \sgn(\E) & \rightarrow & \epsilon_{\Delta} = - \P_{\Delta}(\H + \F) \label{eqn:c3} \\
|\P_{\Delta^c}(\Q)|_{\infty} < \lambda & \rightarrow & \xi(\X)(1 + \|\epsilon_{T}\|) < \lambda \label{eqn:c4}
\end{eqnarray}
Below, we will first show that there exist solutions $\epsilon_T \in T$ and $\epsilon_{\Delta}$ that satisfy conditions (\ref{eqn:c1}) and (\ref{eqn:c3}). We will then bound $\|\epsilon_{\Omega}\|_{\infty}$, $\|\epsilon_T\|$, $\|\P_{T^{\perp}}(\H)\|$, and $\|\P_{T^{\perp}}(\F)\|$ to show that with sufficiently small $\mu(\E)$ and $\xi(\X)$, and appropriately chosen $\lambda$, conditions (\ref{eqn:c2}) and (\ref{eqn:c4}) can be satisfied as well.

First, we show the existence of $\epsilon_{\Delta}$ and $\epsilon_T$ that obey the relationships in (\ref{eqn:c1}) and (\ref{eqn:c3}). It is equivalent to show that there exists $\epsilon_T$ that satisfies the following relation
\[
\epsilon_T = -\P_T(\lambda\ \sgn(\E)) + \P_T\P_{\Delta}(\H) + \P_T\P_{\Delta}\P_T(\P_T\P_{\Omega}\P_T)^{-1}(\epsilon_T)
\]
or
\[
\P_T\P_{\Omega\setminus\Delta}\P_T(\P_T\P_{\Omega}\P_T)^{-1}(\epsilon_T) = -\P_T(\lambda\ \sgn(\E)) + \P_T\P_{\Delta}(\H),
\]
where $\Omega\setminus\Delta$ indicates the complement set of set $\Delta$ in $\Omega$ and $|\Omega\setminus\Delta|$ denotes its cardinality. Similar to the previous argument, when $|\Omega\setminus\Delta| = m_1 - m_2 > m_0$, with a probability $1 - 3n^{-\beta}$, $\P_T\P_{\Omega\setminus\Delta}\P_T(\Z): T \mapsto T$ is an one to one mapping, and therefore $(\P_T\P_{\Omega\setminus\Delta}\P_T(\Z))^{-1}$ is well defined. Using this result, we have the following solution to the above equation
\begin{eqnarray*}
\epsilon_T & = & \P_T\P_{\Omega}\P_T(\P_T\P_{\Omega\setminus\Delta}\P_T)^{-1}\left(-\P_T(\lambda\ \sgn(\E)) + \P_T\P_{\Delta}(\H) \right)
\end{eqnarray*}

We now bound $\|\epsilon_T\|$ and $\|\epsilon_{\Delta}\|_{\infty}$. Since $\|\epsilon_T\| \leq \|\epsilon_T\|_F$, we bound $\|\epsilon_T\|_F$ instead. First, according to Corollary 3.5 in~\cite{candes2010power}, when $\beta = 4$, with a probability $1 - n^{-3}$, for any $\Z \in T$, we have
\[
    \left\|\P_{T^{\perp}}\P_{\Omega}\P_T(\P_T\P_{\Omega}\P_T)^{-1}(\Z)\right\|_F \leq \|\Z\|_F.
\]
Using this result, we have
\begin{eqnarray*}
    \|\epsilon_{\Delta}\|_{\infty} & \leq & \xi(\X)\left(\|\H\| + \|\F\|\right) \\
    & \leq & \xi(\X)\left(1 + \|\P_{T^{\perp}}(\H)\|_F + \|\epsilon_T\| + \|\P_{T^{\perp}}(\F)\|_F\right) \\
    & \leq & \xi(\X)\left(2 + \|\epsilon_T\| + \|\epsilon_T\|_F\right) \\
    & \leq & \xi(\X)\left[2 + (2k+1)\|\epsilon_T\|\right]
\end{eqnarray*}
In the last step, we use the fact that $\rank(\epsilon_T) \leq 2k$ if $\epsilon_T \in T$. We then proceed to bound $\|\epsilon_T\|$ as follows
\begin{eqnarray*}
\|\epsilon_T\| & \leq & \mu(\E)\left(\lambda + \|\epsilon_{\Delta}\|_{\infty}\right)
\end{eqnarray*}
Combining the above two inequalities together, we have
\begin{eqnarray*}
\|\epsilon_T\| & \leq & \xi(\X)\mu(\E)(2k + 1)\|\epsilon_T\| + 2\xi(\X)\mu(\E) + \lambda \mu(\E) \\
\|\epsilon_{\Delta}\|_{\infty} & \leq & \xi(\X)\left[2 + (2k+1)\mu(\E)(\lambda + \|\epsilon_{\Delta}\|_{\infty}\right),
\end{eqnarray*}
which lead to
\begin{eqnarray*}
\|\epsilon_T\| & \leq & \frac{\lambda \mu(\E) + 2\xi(\X)\mu(\E)}{1 - (2k+1)\xi(\X)\mu(\E)} \\
\|\epsilon_{\Delta}\|_{\infty} & \leq & \frac{2\xi(\X) + (2k+1)\lambda\xi(\X)\mu(\E)}{1 - (2k+1)\xi(\X)\mu(\E)}
\end{eqnarray*}
Using the bound for $\|\epsilon_{\Delta}\|_{\infty}$ and $\|\epsilon_T\|$, we now check the condition (\ref{eqn:c2})
\begin{eqnarray*}
1 & > & \mu(\E)\left(\lambda + |\epsilon_{\Delta}|_{\infty}\right) + \frac{1}{2} + \frac{k}{2}\|\epsilon_T\|
\end{eqnarray*}
or
\[
\lambda < \frac{1 - \xi(\X)\mu(\E)(4k + 5)}{\mu(\E)(k+2)}
\]
For the condition (\ref{eqn:c4}), we have
\[
    \lambda > \xi(\X) + \xi(\X)\|\epsilon_T\|
\]
or
\[
    \lambda > \frac{\xi(\X) - (2k - 1)\xi^2(\X)\mu(\E)}{1 - 2(k+1)\xi(\X)\mu(\E)}
\]
To ensure that there exists $\lambda \geq 0$ satisfies the above two conditions, we have
\[
1 - 5(k+1)\xi(\X)\mu(\E) + (10k^2+21k+8)[\xi(X)\mu(\E)]^2 > 0
\]
and
\[
1 - \xi(\X)\mu(\E)(4k + 5) \geq 0
\]
Since the first condition is guaranteed to be satisfied for $k \geq 1$, we have
\[
    \xi(\X)\mu(\E) \leq \frac{1}{4k + 5}.
\]
Thus we finish the proof.
\end{proof}

\section*{Appendix D: Data Statistics}
We listed the detailed domains of the sentiment analysis tasks in Table \ref{tab:sa_data}. We removed the \emph{musical\_instruments} and \emph{tools\_hardware} domains from the original data because they have too few labeled examples. The statistics for the 10 target tasks of intent classification in Table \ref{tab:nlu_data}.

\begin{table*}[th]
\centering
\vspace{0.1 in}
\begin{tabular}{|c|c|c|c|}
\hline
{\bf Domains} & {\bf \#train} &{\bf \#validation}  & {\bf \#test}  \\ \hline
apparel & 7398 & 926 & 928 \\
automotive & 601 & 69 & 66 \\
baby & 3405 & 437 & 414 \\
beauty & 2305 & 280 & 299 \\
books & 19913 & 2436 & 2489 \\
camera\_photo & 5915 & 744 & 749 \\
cell\_phones\_service & 816 & 109 & 98 \\
computer\_video\_games & 2201 & 274 & 296 \\
dvd & 19961 & 2624 & 2412 \\
electronics & 18431 & 2304 & 2274 \\
gourmet\_food & 1227 & 182 & 166 \\
grocery & 2101 & 268 & 263 \\
health\_personal\_care & 5826 & 687 & 712 \\
jewelry\_watches & 1597 & 188 & 196 \\
kitchen\_housewares & 15888 & 1978 & 1990 \\
magazines & 3341 & 427 & 421 \\
music & 20103 & 2463 & 2510 \\
office\_products & 337 & 54 & 40 \\
outdoor\_living & 1321 & 143 & 135 \\
software & 1934 & 254 & 202 \\
sports\_outdoors & 4582 & 566 & 580 \\
toys\_games & 10634 & 1267 & 1246 \\
video & 19941 & 2519 & 2539 \\
\hline
\end{tabular}
\caption{Statistics of the Multi-Domain Sentiment Classification Data.}\label{tab:sa_data}
\end{table*}

\begin{table*}[th]
\centering
\vspace{0.1 in}
\begin{tabular}{|c|c|c|}
\hline
{\bf Dataset ID} & {\bf \#labeled instances} &{\bf \#labels}   \\ \hline
1	&	497	&	11	\\
2	&	3071	&	14	\\
3	&	305	&	21	\\
4	&	122	&	7	\\
5	&	110	&	11	\\
6	&	126	&	12	\\
7	&	218	&	45	\\
8	&	297	&	10	\\
9	&	424	&	4	\\
10	&	110	&	17	\\
\hline
\end{tabular}
\caption{Statistics of the User Intent Classification Data.}\label{tab:nlu_data}
\end{table*}

%% file: naaclhlt2018.bbl
\begin{thebibliography}{30}
\expandafter\ifx\csname natexlab\endcsname\relax\def\natexlab#1{#1}\fi

\bibitem[{Andreas et~al.(2016)Andreas, Rohrbach, Darrell, and
  Klein}]{andreas2016neural}
Jacob Andreas, Marcus Rohrbach, Trevor Darrell, and Dan Klein. 2016.
\newblock Neural module networks.
\newblock In \emph{Proceedings of the IEEE Conference on Computer Vision and
  Pattern Recognition}, pages 39--48.

\bibitem[{Barzilai and Crammer(2015)}]{barzilai2015convex}
Aviad Barzilai and Koby Crammer. 2015.
\newblock Convex multi-task learning by clustering.
\newblock In \emph{AISTATS}.

\bibitem[{Blitzer et~al.(2007)Blitzer, Dredze, and
  Pereira}]{blitzer2007biographies}
John Blitzer, Mark Dredze, and Fernando Pereira. 2007.
\newblock Biographies, bollywood, boom-boxes and blenders: Domain adaptation
  for sentiment classification.
\newblock In \emph{ACL}, volume~7, pages 440--447.

\bibitem[{Cand{\`e}s and Tao(2010)}]{candes2010power}
Emmanuel~J Cand{\`e}s and Terence Tao. 2010.
\newblock The power of convex relaxation: Near-optimal matrix completion.
\newblock \emph{IEEE Transactions on Information Theory}, 56(5):2053--2080.

\bibitem[{Chandrasekaran et~al.(2011)Chandrasekaran, Sanghavi, Parrilo, and
  Willsky}]{chandrasekaran2011rank}
Venkat Chandrasekaran, Sujay Sanghavi, Pablo~A Parrilo, and Alan~S Willsky.
  2011.
\newblock Rank-sparsity incoherence for matrix decomposition.
\newblock \emph{SIAM Journal on Optimization}, 21(2):572--596.

\bibitem[{Chang et~al.(2013)Chang, Qi, Tang, Tian, Rui, and
  Huang}]{chang2013multimedia}
Shiyu Chang, Guo-Jun Qi, Jinhui Tang, Qi~Tian, Yong Rui, and Thomas~S Huang.
  2013.
\newblock Multimedia lego: Learning structured model by probabilistic logic
  ontology tree.
\newblock In \emph{Data Mining (ICDM), 2013 IEEE 13th International Conference
  on}, pages 979--984. IEEE.

\bibitem[{Cheng et~al.(2017)Cheng, Yu, Guo, and Zhou}]{cheng2017metametric}
Yu~Cheng, Mo~Yu, Xiaoxiao Guo, and Bowen Zhou. 2017.
\newblock Few-shot learning with meta metric learners.
\newblock In \emph{NIPS 2017 Workshop on Meta-Learning}.

\bibitem[{Collobert et~al.(2011)Collobert, Weston, Bottou, Karlen, Kavukcuoglu,
  and Kuksa}]{collobert2011natural}
Ronan Collobert, Jason Weston, L{\'e}on Bottou, Michael Karlen, Koray
  Kavukcuoglu, and Pavel Kuksa. 2011.
\newblock Natural language processing (almost) from scratch.
\newblock \emph{Journal of Machine Learning Research}, 12(Aug):2493--2537.

\bibitem[{Crammer and Mansour(2012)}]{crammer2012learning}
Koby Crammer and Yishay Mansour. 2012.
\newblock Learning multiple tasks using shared hypotheses.
\newblock In \emph{Advances in Neural Information Processing Systems}, pages
  1475--1483.

\bibitem[{Finn et~al.(2017)Finn, Abbeel, and Levine}]{finn2017model}
Chelsea Finn, Pieter Abbeel, and Sergey Levine. 2017.
\newblock Model-agnostic meta-learning for fast adaptation of deep networks.
\newblock \emph{arXiv preprint arXiv:1703.03400}.

\bibitem[{Johnson and Zhang(2016)}]{johnson2016supervised}
Rie Johnson and Tong Zhang. 2016.
\newblock Supervised and semi-supervised text categorization using one-hot lstm
  for region embeddings.
\newblock \emph{stat}, 1050:7.

\bibitem[{Joulin et~al.(2016)Joulin, Grave, Bojanowski, and
  Mikolov}]{joulin2016bag}
Armand Joulin, Edouard Grave, Piotr Bojanowski, and Tomas Mikolov. 2016.
\newblock Bag of tricks for efficient text classification.
\newblock \emph{arXiv preprint arXiv:1607.01759}.

\bibitem[{Kang et~al.(2011)Kang, Grauman, and Sha}]{kang2011learning}
Zhuoliang Kang, Kristen Grauman, and Fei Sha. 2011.
\newblock Learning with whom to share in multi-task feature learning.
\newblock In \emph{Proceedings of the 28th International Conference on Machine
  Learning (ICML-11)}, pages 521--528.

\bibitem[{Kim(2014)}]{kim:2014:EMNLP2014}
Yoon Kim. 2014.
\newblock Convolutional neural networks for sentence classification.
\newblock In \emph{EMNLP}, pages 1746--1751, Doha, Qatar. Association for
  Computational Linguistics.

\bibitem[{Koch(2015)}]{koch2015siamese}
Gregory Koch. 2015.
\newblock \emph{Siamese neural networks for one-shot image recognition}.
\newblock Ph.D. thesis, University of Toronto.

\bibitem[{Kumar and Daume~III(2012)}]{kumar2012learning}
Abhishek Kumar and Hal Daume~III. 2012.
\newblock Learning task grouping and overlap in multi-task learning.
\newblock In \emph{Proceedings of the 29th International Conference on Machine
  Learning (ICML-12)}.

\bibitem[{Lake et~al.(2011)Lake, Salakhutdinov, Gross, and
  Tenenbaum}]{lake2011one}
Brenden~M Lake, Ruslan Salakhutdinov, Jason Gross, and Joshua~B Tenenbaum.
  2011.
\newblock One shot learning of simple visual concepts.
\newblock In \emph{CogSci}, volume 172, page~2.

\bibitem[{Lake et~al.(2015)Lake, Salakhutdinov, and Tenenbaum}]{lake2015human}
Brenden~M Lake, Ruslan Salakhutdinov, and Joshua~B Tenenbaum. 2015.
\newblock Human-level concept learning through probabilistic program induction.
\newblock \emph{Science}, 350(6266):1332--1338.

\bibitem[{Li et~al.(2006)Li, Fergus, and Perona}]{li2006one}
Fei-Fei Li, Rob Fergus, and Pietro Perona. 2006.
\newblock One-shot learning of object categories.
\newblock \emph{IEEE Transactions on Pattern Analysis and Machine
  Intelligence}, 28(4):594--611.

\bibitem[{Miller et~al.(2000)Miller, Matsakis, and Viola}]{miller2000learning}
Erik~G Miller, Nicholas~E Matsakis, and Paul~A Viola. 2000.
\newblock Learning from one example through shared densities on transforms.
\newblock In \emph{Computer Vision and Pattern Recognition, 2000. Proceedings.
  IEEE Conference on}, volume~1, pages 464--471. IEEE.

\bibitem[{Mishra et~al.(2017)Mishra, Rohaninejad, Chen, and
  Abbeel}]{mishra2017simple}
Nikhil Mishra, Mostafa Rohaninejad, Xi~Chen, and Pieter Abbeel. 2017.
\newblock A simple neural attentive meta-learner.
\newblock In \emph{NIPS 2017 Workshop on Meta-Learning}.

\bibitem[{Munkhdalai and Yu(2017)}]{munkhdalai2017meta}
Tsendsuren Munkhdalai and Hong Yu. 2017.
\newblock Meta networks.
\newblock \emph{arXiv preprint arXiv:1703.00837}.

\bibitem[{Murugesan et~al.(2017)Murugesan, Carbonell, and
  Yang}]{murugesan2017co}
Keerthiram Murugesan, Jaime Carbonell, and Yiming Yang. 2017.
\newblock Co-clustering for multitask learning.
\newblock \emph{arXiv preprint arXiv:1703.00994}.

\bibitem[{Ng et~al.(2002)Ng, Jordan, and Weiss}]{ng2002spectral}
Andrew~Y Ng, Michael~I Jordan, and Yair Weiss. 2002.
\newblock On spectral clustering: Analysis and an algorithm.
\newblock In \emph{Advances in neural information processing systems}, pages
  849--856.

\bibitem[{Pennington et~al.(2014)Pennington, Socher, and
  Manning}]{pennington2014glove}
Jeffrey Pennington, Richard Socher, and Christopher~D Manning. 2014.
\newblock Glove: Global vectors for word representation.
\newblock In \emph{EMNLP}, volume~14, pages 1532--1543.

\bibitem[{Ravi and Larochelle(2017)}]{ravi2017optimization}
Sachin Ravi and Hugo Larochelle. 2017.
\newblock Optimization as a model for few-shot learning.
\newblock In \emph{International Conference on Learning Representations},
  volume~1, page~6.

\bibitem[{Snell et~al.(2017)Snell, Swersky, and Zemel}]{snell2017prototypical}
Jake Snell, Kevin Swersky, and Richard~S Zemel. 2017.
\newblock Prototypical networks for few-shot learning.
\newblock \emph{arXiv preprint arXiv:1703.05175}.

\bibitem[{Triantafillou et~al.(2017)Triantafillou, Zemel, and
  Urtasun}]{triantafillou2017few}
Eleni Triantafillou, Richard Zemel, and Raquel Urtasun. 2017.
\newblock Few-shot learning through an information retrieval lens.
\newblock In \emph{Advances in Neural Information Processing Systems}, pages
  2252--2262.

\bibitem[{Vinyals et~al.(2016)Vinyals, Blundell, Lillicrap, Wierstra
  et~al.}]{vinyals2016matching}
Oriol Vinyals, Charles Blundell, Tim Lillicrap, Daan Wierstra, et~al. 2016.
\newblock Matching networks for one shot learning.
\newblock In \emph{Advances in Neural Information Processing Systems}, pages
  3630--3638.

\bibitem[{Wang et~al.(2017)Wang, Ramanan, and Hebert}]{NIPS2017_7278}
Yu-Xiong Wang, Deva Ramanan, and Martial Hebert. 2017.
\newblock Learning to model the tail.
\newblock In \emph{Advances in Neural Information Processing Systems 30}, pages
  7032--7042.

\end{thebibliography}
